\documentclass{article}

\usepackage{graphicx} 
\usepackage{subfigure} 
\usepackage{natbib}
\usepackage{algorithm}
\usepackage{algorithmic}
\usepackage{hyperref}

\usepackage{amsmath,amssymb,amsthm,graphicx,multirow}

\newcommand{\intg}[4]{\int_{#3}^{#4} \! #1 \; d #2 \,}
\newcommand{\prob}[1]{P\left( #1 \right) }

\newcommand{\ex}[2]{E\!\left\{ #1 \,\middle|\, #2 \right\} }
\newcommand{\exn}[1]{E\!\left\{  #1 \right\} }

\newcommand{\bias}[1]{\mbox{Bias}\!\left( #1 \right)}
\newcommand{\biasc}[2]{\mbox{Bias}\!\left( #1 \,\middle|\, #2 \right)}
\newcommand{\var}[1]{\mbox{Var}\left( #1 \right)}

\newcommand{\reals}{\mathbb{R}}

\newcommand{\Max}{\mathcal{M}}

\newcommand{\eq}[1]{(\ref{#1})}
\newcommand{\est}{\hat{\mu}}

\newcommand{\X}{\mathcal{X}}
\newcommand{\Opt}{\mathcal{O}}

\newcommand{\indicator}{\mathcal{I}}

\newcommand{\estf}{\hat{f}}
\newcommand{\estF}{\hat{F}}

\newcommand{\msa}{\est^{\mbox{\tiny ME}}_*}

\newcommand{\cv}{\est^{\mbox{\tiny CV}}_*}
\newcommand{\lvcv}{\est^{\mbox{\tiny LVCV}}_*}
\newcommand{\lbcv}{\est^{\mbox{\tiny LBCV}}_*}
\newcommand{\argset}{\hat{a}}
\newcommand{\valset}{\hat{v}}

\newtheorem{theorem}{Theorem}
\newtheorem{conjecture}{Conjecture}

\newtheorem{corollary}{Corollary}

\begin{document} 

\title{Estimating the Maximum Expected Value: An Analysis of (Nested) Cross
Validation and the Maximum Sample Average}

\author{Hado van Hasselt \\ H.van.Hasselt@cwi.nl}

\maketitle 
\begin{abstract} 
We investigate the accuracy of the two most common
estimators for the maximum expected value of
a general set of random variables: a generalization
of the maximum sample average, and cross validation.
No unbiased estimator exists and we show that it is
non-trivial to select a good estimator without knowledge
about the distributions of the random variables.
We investigate and bound the bias and variance of the
aforementioned estimators and prove consistency.
The variance of cross validation can be
significantly reduced, but not without risking
a large bias. The bias and variance of
different variants of cross validation are shown to be very
problem-dependent, and a wrong choice can
lead to very inaccurate estimates.
\end{abstract} 

\section{Introduction}\label{sec_intro}
We often need to estimate the \emph{maximum expected value} of a set of random variables (RVs), when only noisy estimates for each of the variables are given.\footnote{Without loss of generality, we assume that we want to \emph{maximize} rather than \emph{minimize}.}
For instance, this problem arises in
optimization in stochastic decision processes and in algorithmic evaluation.

Formally, we consider a finite set $V = \{ V_1, \ldots, V_M \}$ of $M \geq 2$ independent RVs with finite means $\mu_1, \ldots, \mu_M$ and variances $\sigma^2_1, \ldots, \sigma^2_M$. We want to find the value of $\mu_*(V)$, defined by
\begin{equation}\label{maxE}
\mu_*(V) \equiv \max_i \mu_i \equiv \max_i E \{ V_i \} \enspace.
\end{equation}
We assume the distribution of each $V_i$ is unknown, but a set of noisy samples $X$ is given. The question is how best to use the samples to construct an estimate $\est_*(X) \approx \mu_*(V)$. We write $\mu_*$ and $\est_*$ when $V$ and $X$ are clear from the context. It is easy to construct consistent estimators, but we are also interested in the quality for small sample sizes. The mean squared error (MSE) is the most common metric for the quality of an estimator, but sometimes (the sign of) the bias is more important. Unfortunately, as we discuss in Section \ref{sec_bias}, no unbiased estimators for $\mu_*$ exist.

A common estimator is the \emph{maximum estimator} (ME), which constructs estimates $\est_i \approx \mu_i$ and then uses $\est_* \equiv \max_i \est_i$. When $X_i \subset X$ contains direct samples for $V_i$, and $\est_i$ is the average of $X_i$, the ME is simply the maximum sample average.
The ME on average overestimates $\mu_*$. This bias has been rediscovered several times, for instance in economics \cite{Steen:2004} and decision making \cite{Smith:2006}. It can lead to overestimation of the performance of algorithms \cite{Varma:2006,Cawley:2010}, and poor policies in reinforcement learning algorithms \cite{vanHasselt:2011,vanHasselt:2011phd}.
It is related to \emph{over-fitting} in model selection, \emph{selection bias} in sample selection \cite{Heckman:1979} and feature selection \cite{Ambroise:2002}, and the \emph{winner's curse} in auctions \cite{CapenClappCampbell:1971}. 

The most common alternative to avoid this
bias is \emph{cross validation} (CV) \cite{Larson:1931,Mosteller:68}. If CV is used to construct each $\est_i$, and thereafter to estimate $\mu_*$ (as described in Section \ref{sec_CV}), this is called nested CV or ``double cross'' \cite{Stone:74}. Unfortunately, (nested)
CV can lead to a large variance. Perhaps surprisingly,
we show the absolute bias of CV can be larger than
the bias of the ME that we are trying to prevent.
However, the bias of CV is provably negative, which can
be an advantage.

In this paper, we give general distribution-independent bounds for the bias
and variance of the ME and CV. We present a new variant of
CV and show that it is very dependent
on $V$ which CV estimator is most accurate in terms
of MSE. Therefore, it is non-trivial to construct accurate
CV estimators without some knowledge about the
distributions of $V$. We discuss why standard
10-folds CV is often not a bad choice for model selection,
but show that in other settings other estimators may be
much more accurate.

We now discuss two motivating examples to highlight the practical importance of this general topic.

\paragraph{Learning Algorithms}
Many learning algorithms explicitly maximize noisy values to update their internal parameters. For instance, in reinforcement learning \cite{Sutton:98book} the goal is to find strategies that maximize a reward signal in a (sequential) decision task. Inaccurate biased estimators for $\mu_*$ can have adverse effects on the speed of learning and the strategies that are learned
\cite{vanHasselt:2011}.

\paragraph{Evaluation of Algorithms}
Most machine-learning algorithms have tunable parameters. \emph{Internal parameters}, such as the Lagrangian multipliers of a support vector machine (SVM) \cite{Vapnik:95book}, are optimized by the algorithm. \emph{Hyper-parameters}, such as the choice of kernel function in a SVM, are often tuned manually or chosen with domain knowledge.
Other relevant choices by the experimenter---such as which algorithms to consider and the representation of the problem---can be summarized as \emph{meta-parameters}. 

Typically, we evaluate a set $C$ of configurations, where each $c_i \in C$ denotes a specific algorithm with specific hyper- and meta-parameters. Often, each evaluation is noisy, due to (pseudo-)randomness in the algorithms or inherent randomness in the problem. The performance of $c_i$ is then a random variable $V_i$, and we want to construct an estimate $\est_*$ for the best performance $\mu_*$ñ.\footnote{Sometimes we are more interested in the configuration that optimizes the performance than in the resulting performance, but often the performance itself is at least as important. In part, this depends on whether the focus of the research is on the algorithms or on the problem.}

For instance, \citet{Varma:2006} note that the ME results in
overly-optimistic prediction errors and propose to use nested CV. They evaluate an SVM for various hyper-parameters on an artificial problem, with an actual error of 50\%. The estimate by the ME is 41.7\% and nested CV results in $54.2\%$. Varma and Simon argue that the latter exceeds 50\% because nested CV removes a sample from each training set. However, in fact the difference between 50\% and $54.2\%$ is a demonstration of a completely different general bias that we discuss in Section \ref{sec_CV}.
This bias has received very little attention, although---as we will show---it is not in general smaller than the bias caused by using the ME.

\paragraph{Overview}
In the rest of this section, we discuss related work and (notational) preliminaries.
In Section \ref{sec_bias}, we discuss the bias of estimators in general. In Section \ref{sec_estimators}, we discuss the properties of the ME and of CV, including bounds on their bias and variance. In Section \ref{sec_examples} we discuss concrete settings with empirical
illustrations. Section \ref{sec_disc} contains a discussion and Section \ref{sec_conc}
concludes the paper.

\subsection{Related Work}
The \emph{bootstrap} \cite{Efron:1993} is a resampling method that can be used to estimate the bias of an estimator, in order to reduce this bias. Based on this, Tibshirani and Tibshirani \cite{Tibshirani:2009} propose an estimator for $\mu_*$ for model selection in classification. Inevitably---see Section \ref{sec_nonexistence}---the resulting estimate is still biased, and it is typically more variable that the original estimate. Also specifically considering model selection for classifiers, Bernau et al. \cite{Bernau:2011} propose a smoothed version of nested CV. The resulting estimator performs similar to normal (nested) CV, which in turn is shown to typically be more accurate than the approach by Tibshirani and Tibshirani. In this paper we focus on CV and the ME, which are by far the most widely used.

The problem of estimating $\mu_*$ is related to the \emph{multi-armed bandit} framework \cite{robbins1952some,Berry_bandit:85}, where the goal is to find the \emph{identity} of the action with the maximum expected value rather than its \emph{value}. The focus in the literature on multi-armed bandits is often on how best to collect samples. In contrast, in this paper we assume that a set of samples is given. A discussion on how best to collect samples to minimize the bias or MSE is outside the scope of this paper, although we do note that minimizing online regret \cite{Lai:1985,Auer:2002} does not necessarily correspond to minimizing the online MSE of the estimator.


\subsection{Preliminaries}\label{sec_prelim}

The measurable domain of $V_i$ is $\X_i$, and $f_i : \X_i \to \reals$ denotes its probability density function (PDF), such that $\mu_i = \intg{x\;f_i(x)}{x}{\X_i}{}$. For conciseness, we assume $\X_i = \reals$.
We assume these PDFs $f_i$ are unknown and therefore $\mu_* = \max_i \mu_i$ can not be found analytically.

We write $\est_i(X)$ for an estimator for $\mu_i$ based on a sample set $X$. Similarly, $\est_*(X)$ is an estimator for $\mu_*$. We write $\est_i$ and $\est_*$ when $X$ is clear from the context. If $X_i \subset X$ is a set of unbiased samples for $V_i$, $\est_i$ might be the sample average.
In that case, $\est_i$ is unbiased for $\mu_i$. In general, $\est_i$ can be biased for $\mu_i$. As discussed in the next section, no general unbiased estimator for $\mu_*$ exists, even if all $\est_i$ are unbiased.

The following definitions will be useful below, when stating necessary and sufficient conditions for a strictly positive or a strictly negative bias. The set of \emph{optimal} indices for RVs $V$ is defined as
\begin{equation}\label{chEB_eq_optimal_estimator}
\mathcal{O}(V) \equiv \left\{ i \,\left|\, \mu_i = \mu_* \right. \right\} \enspace.
\end{equation}

The set of \emph{maximal} indices for samples $X$ is defined as
\begin{equation}\label{max_indices}
\Max(X) \equiv \left\{ i \,\left|\, \est_i(X) = \max_j \est_j(X) \right. \right\} \enspace.
\end{equation}
An estimator is called optimal or maximal whenever its index is optimal or maximal, respectively. Note that optimal estimators are not necessarily maximal and maximal estimators are not necessarily optimal.

\section{The Bias of an Estimator}\label{sec_bias}
Let $\mathcal{V}$ be a function space containing all admissible sets of $M$ RVs. We might know $\mathcal{V}$, but not the precise identity of $V \in \mathcal{V}$. For instance, $\mathcal{V}$ may be the set of all sets of $M$ normal RVs with finite moments.
Let $p : \mathcal{V} \to \reals$ be a PDF over $\mathcal{V}$. 
The expected MSE of an estimator $\est_*$ is equal to
\begin{equation}\label{MSEV}
\intg{ \! p(V) \intg{ \! P(X|V) \left( \est_*(X) - \mu_* \right)^2 }{X\!}{X}{} }{ V\! }{ \mathcal{V} }{} ,
\end{equation}
In any given concrete setting, there is a single unknown set $V$. Therefore, $p$ does not exist `in the world'. Rather, $p$ might model our prior belief about which sets $V$ are likely in a given setting, or it might specify the $V$ for which we would like an estimator to perform well.
The MSE consists of variance and bias. 
To reason in some generality about which estimators are good in practice, we discuss the \emph{non-existence of unbiased estimators} and the \emph{direction of the bias}.

\paragraph{Non-Existence of Unbiased Estimators}\label{sec_nonexistence}
By definition, $\est_*$ is a general unbiased estimator (GUE) for $\mathcal{V}$ if and only if
\begin{equation}\label{GUE}
\forall V \in \mathcal{V} :  \ex{\est_* }{ V } = \mu_* \enspace.
\end{equation}
Unfortunately, for most $\mathcal{V}$ of interest no such estimator exists.
For instance, \citet{Blumenthal:1968} show no GUE exists for two normal distributions  and \citet{Dhariyal:1985} proved this for arbitrary $M\geq 2$ and for more general distributions, including the exponential family. Essentially, the argument is that a reasonable estimator for $\mu_*$ depends smoothly on the values of the samples, whereas the real value $\mu_*$ is a piece-wise linear function with a discontinuous derivative. We can not know the location of these discontinuities without knowing the actual maximum.

Note that \eq{GUE} is already false if $\mathcal{V}$ contains only a single set of variables for which $\est_*$ is biased. 
However, bias alone does not tell us everything, and a low bias does not necessarily imply a small expected MSE.

\paragraph{The Direction of the Bias}
In some cases, the direction of the bias is very important.
Suppose we test an algorithm for various hyper-parameters and observe that the best performance is better than some baseline. If we simply use the highest test result, it can not be concluded that the algorithm can really structurally outperform the baseline for any of the specific hyper-parameters. Although this may sound trivial, it is common in practice: when we manually tune hyper- or meta-parameters on a problem and use the best result, we are using $\max_i \est_i$, which has non-negative bias. It is hard to avoid optimizing on meta-parameters: these include the very (properties of the) problem we test the algorithm on.

The practical implication of this positive bias is that the algorithm will disappoint in future evaluations on similar (real-world) problems.
In contrast, if we use an estimator with \emph{non-positive bias} and our estimate is higher than the baseline, we can have much more confidence that the algorithm can reach that performance consistently with a properly tuned hyper-parameter.
This is similar to the considerations about overfitting in model selection, where CV is most often used. We prove below that CV indeed has non-positive bias, and can therefore avoid overestimations of $\mu_*$.

As another example, the performance of most machine-learning algorithms improves when more data is available. When the data collection is expensive it is useful to predict how an algorithm performs when more data is available, before actually collecting this data. An overestimation of the future performance can lead to a misallocation of resources, since the collected data may be less useful than predicted. An underestimation means we may be too pessimistic, and too often decide not to collect more data. Whether or not the false positives are more important than the false negatives depends crucially on specifics of the setting.

\section{Estimators for the Maximum Expected Value}\label{sec_estimators}
In this section, we discuss the ME and CV estimators for $\mu_*$ in detail. We bound the biases and variances of all estimators, discuss similarities and contrasts, and prove consistency. We introduce a low-variance variant of CV. We give necessary and sufficient conditions for non-zero biases to occur for all estimators, and perhaps surprisingly we show that there are settings in which the negative bias of all variants of CV is larger in size than the positive bias of ME. All proofs are given in an appendix and the end of this paper.

\subsection{The Maximum-Estimator Estimator}\label{sec_ME}
The maximum-estimator (ME) estimator for $\mu_*$ is
\begin{equation}\label{ME_est}
\msa(X) \equiv \max_i \est_i(X) \enspace,
\end{equation}
where $\est_i$ is a (possibly biased) estimator for $\mu_i$.
Because it is conceptually simple and easy to implement, the ME estimator is often used in practice.
The theorem below proves its bias is non-negative and gives necessary and sufficient conditions for a strictly positive bias. The theorem is stronger and more general than some similar earlier theorems. For instance, \citet{Smith:2006} do not consider the possibility of multiple optimal variables, and do not discuss necessity of the conditions for a strictly positive bias.

\begin{theorem}\label{theorem_overest}
For any given set $V$, $M\geq 1$ and unbiased estimators $\est_i$, $\ex{\est_i}{V} = \mu_i$,
\[ 
\ex{ \msa }{ V }
\geq \mu_* \enspace,
\] 
with equality if and only if all optimal indices are maximal with probability one.
\end{theorem}

Theorem \ref{theorem_overest} implies a lower bound of zero for the bias of the ME.
An upper bound for arbitrary means and variances is given by \citet{Aven:1985}:
\begin{equation}\label{ME_bound}
\bias{\msa} \leq \sqrt{\frac{M-1}{M} \sum_{i}^M \var{ \est_i }} \enspace,
\end{equation}
which is tight when the estimators are iid \citep{Arnold:1979}, indicating that iid variables are a worst-case setting.

We do not know of previous work that bounds the variance, which we discuss next.
\begin{theorem}\label{ME_variance_bound}
The variance of the ME estimator is bounded by
$
\var{ \msa } \leq \sum_{i=1}^M \var{ \est_i }
$.
\end{theorem}

Theorem \ref{ME_variance_bound} and bound \eq{ME_bound} imply that $\msa$ is consistent for $\mu_*$ whenever each $\est_i$ is consistent for $\mu_i$ and that
$
\mbox{MSE}(\msa) < 2 \sum_{i=1}^M \var{ \est_i }
$.

\subsection{The Cross-Validation Estimator}\label{sec_CV}
In general, $\mu_*$ can be considered to be a weighted average of the means of all optimal variables:
$
\mu_* = \frac{1}{|\Opt(V)|}\sum_{i=1}^M \indicator(i \in \Opt(V)) \mu_i 
$, 
where $\indicator$ is the indicator function.
We do not know $\Opt(V)$ and $\mu_i$, but with sample sets $A,B$ we can approximate these with $\Max(A)$ and $\est_i(B)$ to obtain
\begin{equation}\label{maximal_weights}
\mu_* \approx \est_* \equiv \frac{1}{|\Max(A)|}\sum_{i=1}^M \indicator(i \in \Max(A)) \est_i(B) \enspace.
\end{equation}
If $A = B = X$, this reduces to $\msa$. However, suppose that $A$ and $B$ are independent.
This idea leads to \emph{cross-validation} (CV) estimators.
Of course, CV itself is not new. 
However, it seems to be
less well-known how properties of the problem affect
the accuracy and that CV can be quite biased.

We split each $X$ into $K$ disjoint sets $X^k$ and define $\est^k_i \equiv \est_i(X^k)$. For instance, $\est^k_i$ might be the sample average of $X^k_i$.
We consider two different CV estimators. In both methods, for each $k \in \{ 1, \ldots, K \}$ we construct an \emph{argument set} $\argset^k$ and a \emph{value set} $\valset^k$.

\emph{Low-bias cross validation} (LBCV) is the `standard' CV estimator, where $K-1$ sets are used to build the argument set $\argset^k$ (the model), and the remaining set is used to determine its value:
\begin{align*}
\argset^k_i &\equiv \est_i(X \setminus X^k) \quad\mbox{and}
&\valset^k_i &\equiv \est_i(X^k) \equiv \est^k_i \enspace. 
\end{align*}
\emph{Low-variance cross validation} (LVCV) reverses the definitions for $\argset^k$ and $\valset^k$:
\begin{align*}
\argset^k_i &\equiv \est(X_i^k) \equiv \est^k_i \quad\mbox{and}
&\valset^k_i &\equiv \est(X_i \setminus X_i^k) \enspace. 
\end{align*}
We do not know of any previous work that discusses this variant. However, its lower variance can sometimes result in much lower MSEs than obtained by LBCV. For both LBCV and LVCV, if $\est_i(X)$ is the sample average of $X_i$ and all samples are unbiased, then $E \{ \argset^k_i \} = E \{ \valset^k_i \} = \mu_i$.

For either approach, $\Max^k$ is the set of indices that maximize the argument vector.
For LBCV this implies $\Max^k = \Max(X \setminus X^k)$ and for LVCV this implies $\Max^k = \Max(X^k)$.
We find the value of these indices with the value vector, resulting in
\begin{equation}\label{CV_estk}
\est^k_* \equiv \frac{1}{|\Max^k|} \sum_{i \in \Max^k} \valset^k_i
\enspace.\end{equation}
We then average over all $K$ sets:
\begin{equation}\label{CV_est}
\cv \equiv \frac{1}{K} \sum_{k=1}^K \est^k_* =  \frac{1}{K} \sum_{k=1}^K \frac{1}{|\Max^k|} \sum_{i\in \Max^k}^M \valset^k_i \enspace,
\end{equation}
where either $\cv = \lbcv$ or $\cv = \lvcv$, depending on the definitions of $\argset^k_i$ and $\valset^k_i$.

The construction of $\est^k_*$ performs the approximation:
$
\valset^k_i \approx \valset^k_{i_*} \approx \mu_{i_*} \equiv \mu_*
$,
where $i \in \Max^k$ and $i^* \in \Opt$.
The first approximation results from using $\Max^k$ to approximate $\Opt$ and is the main source of bias. The second approximation results from the variance of $\valset^k$.

For large enough $X$, $K$ can be treated as a parameter that trades off bias and variance. For LBCV larger $K$ implies less bias and more variance, while for LVCV it implies more bias and less variance. For $K=2$, LBCV and LVCV are equivalent. If $K>2$, LVCV is more  biased but less variable than LBCV, since $\Max^k$ is then based on fewer samples while $\valset^k_i$ is based on more samples. When $\forall i : |X_i| = K$ for LBCV, $\valset^k_i$ is based on a single sample, resulting in a large variance. This variant is commonly known as \emph{leave-one-out} CV. When $\forall i : |X_i| = K$ for LVCV, $\argset^k_i$ is based on a single sample, potentially resulting in large bias due to large probabilities of selecting sub-optimal indices. 

The bias of CV for $\mu_*$ has received comparatively little attention. Sometimes  the bias is mentioned without explanation \citep{Kohavi:1995}, and sometimes it is even claimed that CV is unbiased \citep{Mannor:2007}. Often, any observed bias is attributed to the fact that $\argset^k$ can be biased when it based on $\frac{K-1}{K}|X|$ rather than $|X|$ samples \citep{Varma:2006}. This can be a factor, but the bias induced by using $\Max^k$ for $\Opt$ is often at least as important, as will be demonstrated below.  Some confusion seems to arise from the fact that $\valset^k_i$ is often unbiased for $\mu_i$. Unfortunately, this does not imply that $\cv$ is unbiased for $\mu_*$.

Next, we prove that CV estimators can have a negative bias even if $\argset$ and $\valset$ are unbiased, and we give necessary and sufficient conditions for a strictly negative bias.
\begin{theorem}\label{CV_underest_theorem}
If $\ex{\est^k_i}{V} = \mu_i$ is unbiased then
$
\ex{\cv}{V}
\leq \mu_*
$ 
is negatively biased, 
with a strict inequality if and only if there is a non-zero probability that any non-optimal index is maximal.
\end{theorem}
The theorem shows that $\lvcv$ and $\lbcv$ on average underestimate $\mu_*$ if and only if there is a non-zero probability that $i \in \Max^k(X)$ for some $i \notin \mathcal{O}(V)$. A prominent case in which this does \emph{not} hold is when all variables have the same mean, since then $i \in \Opt(V)$ for all $i$. Interestingly, this implies that CV is unbiased when the $V_i \in V$ are iid, which is a worst case for the ME.
Theorem \ref{CV_underest_theorem} implies that the bias of CV is bound from above by zero. We conjecture the bias is bound from below as follows.

\begin{conjecture}\label{CV_bound}
Let $\ex{\est^k_i}{V} = \mu_i$.
Then
\[
\bias{\cv} > - \frac{1}{K} \sum_{k=1}^K \sqrt{ \sum_{i=1}^M \var{\argset^k_i} } 
\enspace.\]
\end{conjecture}
We do not prove this conjecture here in full generality, but there is a proof for $M=2$ in the appendix.
It makes intuitive sense that the bias of $\cv$ depends only on the variances $\var{\argset^k_i}$ if each $\est^k_i$ is unbiased.
The bias of the CV estimators is unaffected by the fact that $\cv$ averages over $K$ estimators $\est^k_*$, but $K$ does affect the bias by regulating how many samples are used for each $\argset^k_i$. As mentioned earlier, for LVCV larger a $K$ implies a higher bias since then $\argset^k_i$ is \emph{more} variable, while for LBCV a larger $K$ implies a lower bias since then $\argset^k_i$ is \emph{less} variable.

Although CV is known for low
bias and high variance, the next theorem shows its absolute bias is
not necessarily smaller than the absolute bias of the ME.
\begin{theorem}
There exist $V$ and $N = |X|$ such that $|\biasc{\cv}{N}|>|\biasc{\msa}{N}|$ for any $K$ and for any variant of CV.
\end{theorem}
Two different experiments in Section \ref{sec_examples} prove this theorem, since there even the negative bias of leave-one-out LBCV is larger in size than the positive bias of ME.

\begin{theorem}\label{CV_variance_bound}
The variance of $\lbcv$ is bounded by
\begin{equation*}
\var{ \lbcv } \leq \frac{1}{K^2} \sum_{k=1}^K \sum_{i=1}^M \var{\est^k_i}
\enspace.
\end{equation*}
\end{theorem}
If each $\est^k_i$ is unbiased, the variance of LVCV is necessarily smaller than that of LBCV and the same bound applies trivially to $\lvcv$.
\begin{corollary}\label{CV_variance_av}
If $\est_i$ is the sample average of $X_i$ and $|X^k_i| = |X_i|/K$ for all $k$, then
$
\var{ \cv } \leq \sum_{i=1}^M \var{\est_i} 
$ for both LBCV and LVCV.
\end{corollary}
Conjecture \ref{CV_bound} and Theorem \ref{CV_variance_bound} imply that CV is consistent if each $\est_i$ is consistent and $K$ is fixed (or slowly increasing, see also \citet{Shao:1993}), and that
$
\mbox{MSE}(\cv) \leq 2 \sum_{i=1}^M \var{\est_i}
$.


\section{Concrete Illustrations}\label{sec_examples}
To illustrate that it is non-trivial to select an accurate
estimator, we discuss some concrete examples.

\subsection{Multi-Armed Bandits for Internet Ads}
The framework of multi-armed bandits can be used to optimize
which ad is shown on a website
\citep{Langford:2008,Strehl:2010}. Consider $M$ ads with unknown fixed expected returns per visitor $\mu_i$. Bandit algorithm can be used to balance exploration and exploitation to optimize the online return per visitor, which converges to $\mu_*$. However, quick accurate estimates of $\mu_*$ can be important, for instance to base future investments on. Additionally, placing any ad may induce some cost $c$, so we may want to know quickly whether $\mu_* > c$.

\begin{figure*}[t]
\begin{center}
\includegraphics[scale=0.52]{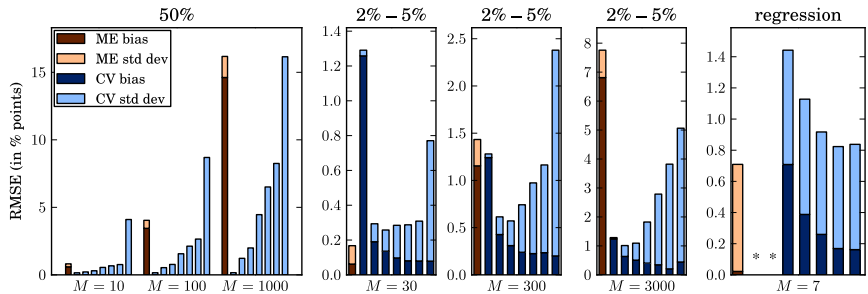}
\caption{\label{fig_ads} The MSE for $\msa$, $\lbcv$ and $\lvcv$ for different settings, averaged over 2,000 experiments. The left-most bar is always $\msa$. The other bars are, from left to right, leave-one-out LVCV, 10-folds LVCV, 5-folds LVCV, 2-folds CV, 5-folds LBCV, 10-folds LBCV and leave-one-out LBCV. Note that 2-folds LVCV is equivalent to 2-folds LBCV, which are therefore not shown separately.}
\end{center}
\end{figure*}

For simplicity, assume each ad has the same return per click, such that only the click rate matters and each $V_i$ can be modeled with a Bernoulli variable with mean $\mu_i$ and variance $(1-\mu_i)\mu_i$. In our first experiment, there are $N=100,000$ visitors, $M=10$, $M=100$ or $M=1000$ ads, and $\forall i: \mu_i = 0.5$. All ads are shown equally often, such that $\forall i: N_i = N/M$. Because all means are equal, Theorem \ref{CV_underest_theorem} implies that CV estimators are unbiased; their MSE depends solely on the variance. In the second---more realistic---setting, the $M$ mean click rates are distributed evenly between $0.02$ and $0.05$, there are $N=300,000$ visitors, and $M=30$, $M=300$, or $M=3000$ ads.

The results are shown in the first four plots in Figure \ref{fig_ads}. We show the root MSE (RMSE), such that the units are percentage points. Within the RMSEs, the contributions of the bias and the variance are shown. 
Note that $\mbox{MSE} = \mbox{bias}^2 + \mbox{variance}$, and therefore $\mbox{RMSE} = \sqrt{\mbox{bias}^2 + \mbox{variance}} \neq \mbox{bias} + \mbox{std dev}$.
This implies that the depicted contributions of bias and variance to the RMSE are not in general exactly equal to the bias and standard deviation, but this depiction does allow us to see directly how many percentage points of error are caused by bias and by variance.

In the first setting (left plot) CV is indeed unbiased. Leave-one-out LVCV has the lowest variance of all CV methods---it is barely visible---which implies it has the smallest MSE. For $M=1000$ ads, the huge bias of the ME causes it to overestimate the actual maximal click rate by more than 15\%.

In the second setting (middle three plots), there is a clear trade-off in CV: LVCV with large $K$ has large bias and small variance, whereas LBCV with large $K$ has small bias and large variance.\footnote{Sometimes the bias of LBCV seems to increase slightly for higher $K$. These are noise-related artifacts.} The bias of the CV estimators is clearly important, even though each $\est^k_i$ is unbiased. Even for leave-one-out LBCV the bias is non-negligible: for $M=30$ its bias is larger than the bias of the ME. Interestingly, when $M$ increases (and the number of samples per ad decreases correspondingly) the error for leave-one-out LVCV stays virtually unchanged, at approximately 1.3\%. Since the error of all other estimators increases with increasing $M$, this implies that leave-one-out LVCV goes from being by far the least accurate for $M=30$ to almost the most accurate for $M=3000$. In contrast, the ME goes from being the most accurate for $M=30$ to the least accurate for $M=3000$. The reason is that for increasing $M$, the variables are relatively more similar to iid variables, which is a best case for LVCV and a worst case for the ME. In all three cases, 10- and 5-folds LVCV are a good choice.

\subsection{Evaluation of Algorithms}
We now consider a regression problem. The goal is to fit polynomials on noisy samples from a function $r(y) = 4( \sin(y) + \sin(2y) )$. Let $X = \{ (y,r(y) + \omega) \,|\, y \in Y \}$ denote a noisy data set for inputs $Y$, where $\omega$ is zero-mean Gaussian noise with variance $\sigma^2_\omega = 4$.  
Let $p_i$ denote a polynomial of degree $i$, of which the coefficients are fitted with least-squares on $X$.

Let $Y = \{ 0, 0.05, \ldots, 3.95, 4 \}$ be 81 equidistant inputs.
We want to maximize the negative MSE. The lowest expected MSE of fitting each $p_i$ on 81 samples
and testing on an independent test set of 81 samples
is obtained at $4.34$ for $i=5$, which implies $\mu_* = \mu_5 = -4.34$.
We construct 1,000 independent noisy sets $X = \{ (y,r(y) + \omega) \,|\, y \in Y \}$. For each $X$, we conduct the following experiment.

For any given $Z \subseteq X$, $\est_i$ is defined by an inner CV loop as follows. For each $z \in Z$, we fit $p_i$ on $Z \setminus \{z\}$ and test the error on $z$ to obtain an error $e_i(z)$. We average these errors to obtain: $\est_i(Z) = \frac{1}{|Z|} \sum_{z \in Z} e_i(z)$.  This implies $\est_i$ is biased, since $p_i$ is fitted on $|Z|-1 < 81$ samples. 
For the ME, $\msa = \max_i \est_i(X)$ which means $|Z| = 80$ samples are used to fit each $p_i$. For LBCV, $\argset^k_i = \est_i(X \setminus X^k)$ which means $|Z|=\frac{K-1}{K}81$. For LVCV, $\argset^k_i = \est_i(X^k)$, which means $|Z| = \frac{1}{K}81$. Since $|Z|$ can then be much smaller than 81, LVCV can be significantly biased. We consider $K\in \{2,3,9,81\}$. When $K=81$, LBCV is also known as \emph{nested leave-one-out CV}. Figure \ref{fig_ads} (right plot) shows the results. 

LVCV is not shown for $K=81$ and $K=9$: LVCV with $K=81$ is meaningless, since one cannot fit a polynomial on a single point. The MSE for $K=9$ is huge. In sharp contrast with the previous settings, LVCV fares poorly---even in terms of variance---and leave-one-out LBCV is the best CV estimator. However,  interestingly the ME is more accurate than all CV estimators, and even the size of its bias ($0.018$) is much smaller than that of $n$-fold LBCV ($-0.190$).

\section{Discussion}\label{sec_disc}
Our results show that it is hard to choose an estimator that is good in general. Unfortunately, the best choice in one setting can be the worst choice in another. A poorly chosen CV estimator can be far less accurate than the ME.  This does not imply that we suggest using the ME; it is often very biased. 

A potential advantage of CV estimators in some settings is a
guaranteed non-positive bias. This can be
desirable even if the estimator is less accurate. However, in our results the recommendation to always use 10-folds LBCV \citep{Kohavi:1995} seems unfounded. When each $\est_i$ is unbiased and especially when $M$ is large, LVCV often performs much better. On the other hand, when each estimator $\est_i$ has a bias that decreases with the number of samples, the bias of LVCV can become prohibitively large, as illustrated in the regression setting. This explains why 10-folds LBCV is often not a bad choice for model selection, as long as $M$ is fairly small and $\est_i$ is fairly biased. However, note that 5- and 10-folds LBCV were the most accurate estimator in \emph{none} of our experiments.

As a general recommendation, it may be good to try both the ME and one or more CV estimators. If the estimates are close together, this indicates they are more likely to be accurate. Although the true maximum expected value will often lie between the estimate by the ME and those by CV, one should not simply average these estimates: as we have shown that for instance the ME can be very biased in some settings, and hardly biased in others. Furthermore, the potentially excessive variance of some variants of LBCV implies that in some cases its estimate may itself be an overestimation, which is why we recommend to include LVCV in the analysis.

\paragraph{Alternative estimators} Of course, there are possible alternatives to the estimators we discussed. First, one can consider using the maximum of some lower confidence bounds on the individual value estimates. Although this does counter the overestimation of ME, it can not be guaranteed that this does not lead to an underestimation in its place. Furthermore, it is non-trivial to select a good confidence interval, and the resulting estimate will typically be much more variable than the ME.

Second, for model-selection there exist criteria such as AIC \citep{Akaike:1974} and BIC \citep{Schwarz:1978} that use a penalty term based on the number of parameters in the model. Obviously, such penalties are only useful when comparing homomorphic models with different numbers of parameters, and therefore do not apply to the more general setting we consider in this paper. Furthermore, the main purpose of these criteria is not to give an accurate estimate of the expected value of the best model, but to increase the probability of selecting it. These goals are related, but unfortunately not equivalent.

Finally, one can estimate belief distributions $\estF_i$ for the location of each $\mu_i$, for instance with Bayesian inference. With these distributions, we can estimate $\mu_*$. This approach is less general, since it requires prior knowledge about $\mathcal{V}$, but then it does seem reasonable. The probability that the maximum mean is smaller than some $x$ is equal to the probability that all means are smaller than $x$. Therefore, its CDF is $\estF_{\max}(x) = \prod_{i=1}^M \estF_i(x)$, which we can use to estimate $\mu_*$. The resulting Bayesian estimator (BE) is
\[
\est_*^{\text{BE}} = \intg{ x \sum_{i=1}^M \estf_i(x) \prod_{j \neq i} \estF_j(x) }{ x }{ -\infty }{ \infty } \enspace,
\]
where $\estf_i(x) = \frac{d}{dx} \estF_i(x)$.
To show a perhaps counter-intuitive result from this approach, we discuss a small example. Consider two Bernoulli variables. We consider all means equally likely and use a uniform prior Beta distribution, with parameters $\alpha = \beta = 1$. Suppose $\mu_1 = \mu_2 = 0.5$. We draw two samples from each variable. The expected estimate for the ME is $\frac{21}{32}$, for a bias of $\frac{5}{32} \approx 0.156$. CV is unbiased, since the means are equal. For the BE $\estF_i(x)$ is $1 - (1-x)^3$, $3x^2 - 2x^3$ or $x^3$, depending on how many samples for $V_i$ are equal to one. Its expected value is then $\exn{ \est_*^{\text{BE}} } = \frac{737}{1120} \approx 0.658$. Note that the positive bias is even higher than the bias of the ME. This is due to our uniform prior: if the prior on the individual variables is uniform, this implies the prior for the maximum expected value is negatively skewed, and its expected value is increased. The effect is already apparent with two variables, but it increases further with the number of variables due to the shape of $\estF_{\max}(x) = \prod_{i=1}^M \estF_i(x)$.

\section{Conclusion}\label{sec_conc}
We analyzed the bias and variance of the two most common estimators for the maximum expected value of a set of random variables. The \emph{maximum estimate} results in non-negative bias. The common alternative of \emph{cross validation} (CV) has non-positive bias, which can be preferable. Unfortunately, the accuracies of different variants of CV are very dependent on the setting; an uninformed choice can result in extremely inaccurate estimates.
No general rule---e.g., always use 10-fold CV---is always optimal.

\subsection*{Appendix}
\begin{proof}[Proof of Theorem \ref{theorem_overest}]
For conciseness, we leave $V$ and $X$ implicit. Let $j \in \Opt$ be an arbitrary optimal index, and define event $A_j \equiv (j \in \Max)$ to be true if and only if $j$ is maximal. We can write
\[
\mu_* = P(A_j) \ex{ \est_j }{ A_j } + P(\neg A_j) \ex{ \est_j }{ \neg A_j } \enspace.
\]
Note: $\ex{ \est_j }{ A_j } = \ex{ \msa }{A_j}$ and $\ex{ \est_j }{ \neg A_j } < \ex{ \msa }{\neg A_j}$. Therefore,
$
\mu_* \leq \exn{ \msa }
$, 
with equality if and only if $\prob{ \neg A_j } = 0$ for all $j \in \Opt$.
\end{proof}

\begin{proof}[Proof of Theorem \ref{ME_variance_bound}]
Let $A$ and $B$ be independent sets of RVs with $\exn{ A_i } = \exn{ B_i }$ and $\exn{ A_i^2 } = \exn{ B_i^2 }$. Define
\[
C^{(i)} \equiv
(A \setminus A_i) \cup \{ B_i \} =
\{ A_1, \ldots, A_{i-1}, B_i, A_{i+1}, \ldots, A_M \}
\enspace.\]
The Efron-Stein inequality \citep{EfronStein:1981} states that for any  $g : \reals^M \to \reals$:
\begin{equation*}
\mbox{Var}\left( g( A  ) \right) \leq \frac{1}{2} \sum_{i=1}^M \exn{ \left( g( A ) - g( C^{(i)} ) \right)^2 } \enspace.
\end{equation*}
Let $A$ and $B$ be independent instantiations of $\est$ and let $g( A ) = \max_i A_i$ for any $A$. We derive
\begin{align*}
& \var{ \msa }
\leq \frac{1}{2} \sum_{i=1}^M \exn{ \left( \max_j A_j - \max_j C^{(i)}_j \right)^2 } \\
& \leq \frac{1}{2} \sum_{i=1}^M \exn{ \left( A_i - B_i \right)^2 } = \sum_{i=1}^M \var{\est_i} \enspace.\qedhere
\end{align*}
\end{proof}

\begin{proof}[Proof of Theorem \ref{CV_underest_theorem}]
Let $w^k_i \equiv \exn{\indicator(i \in \Max^k)/|\Max^k|\,}$. Then $\exn{\est^k_*} = \sum_i^M w^k_i \mu_i \leq \mu_*$, because $\Max^k$ and $\valset^k_i$ are independent, $\sum_i^M w^k_i = 1$ and $\exn{\valset^k_i} = \mu_i$.
Note that $w^k_i > 0$ if and only if $\prob{ i \in \Max^k}>0$. Therefore, $\exn{\est^k_*} < \mu_*$ if and only if there exists a $i \notin \Opt$ such that $\prob{ i \in \Max^k } > 0$.
\end{proof}

\begin{proof}[Proof of Conjecture \ref{CV_bound} for $M=2$]
Assume without loss of generality that $\mu_1 = \mu_*$. The assumption $\exn{\est^k_i} = \mu_i$ implies that $\exn{ \valset^k_i } = \mu_i$. Then,
\begin{align*}
\bias{\est^k_*}
& = \exn{ \frac{ \indicator( 2 \in \Max^k ) }{ |\Max^k| } } \left( \mu_2 - \mu_1 \right) \\
& \geq \prob{ 2 \in \Max^k } \left( \mu_2 - \mu_1 \right) \\
& = \prob{ \argset^k_2 \geq \argset^k_1 } \left( \mu_2 - \mu_1 \right) \\
& \geq \frac{ \left( \var{ \argset^k_1 } + \var{ \argset^k_2 } \right) \left( \mu_2 - \mu_1 \right)}{ \var{ \argset^k_1 } + \var{ \argset^k_2 } + ( \mu_1 - \mu_2 )^2 } \\
& \geq -\frac{1}{2} \sqrt{ \var{ \argset^k_1 } + \var{ \argset^k_2 } } \enspace,
\end{align*}
where the second inequality follows from Cantelli's inequality, and the third inequality is the result of minimizing for $\mu_2 - \mu_1$. From this, it follows that for $M=2$
\[
\bias{ \cv } \leq - \frac{1}{2K} \sum_{k=1}^K \sqrt{ \sum_{i=1}^2 \var{\argset^k_i} } \enspace,
\]
which is a factor $\frac{1}{2}$ tighter than the general bound in the conjecture.
\end{proof}

\begin{proof}[Proof of Theorem \ref{CV_variance_bound}]
We apply definition \eq{CV_est} and use $\sum_{k=1}^K \valset^k_i = \sum_{k=1}^K \est_i^k$ to derive
\begin{align*}
\var{ \!\frac{1}{K} \sum_{k=1}^K \frac{1}{|\Max^k|} \sum_{i\in \Max^k}^M \valset^k_i \!}
& \leq \var{\! \frac{1}{K} \sum_{k=1}^K \sum_{i=1}^M \valset^k_i \!} \\
& \leq 
\frac{1}{K^2} \sum_{k=1}^K \sum_{i=1}^M \var{\est^k_i} \enspace.\qedhere
\end{align*}
\end{proof}

\begin{proof}[Proof of Corollary \ref{CV_variance_av}]
Apply Theorem \ref{CV_variance_bound} with
$
\var{\est^k_i} = \sigma_i^2/|X^k_i| = K \sigma_i^2/|X_i| = K \var{\est_i}
$
\end{proof}

\bibliography{../all.bib}
\bibliographystyle{abbrvnat}

\end{document}